\documentclass[letterpaper,conference]{IEEEtran}

\usepackage[T1]{fontenc}
\usepackage{cite}
\usepackage{microtype}
\usepackage{graphicx}
\usepackage{subcaption}
\usepackage{rotating}
\usepackage{booktabs}
\usepackage{amsmath,amssymb,amsfonts}
\usepackage{mathtools}
\usepackage{tikz}
\usepackage{accents}
\usepackage{siunitx}
\usepackage{standalone}
\usepackage{pgfplots}
\usepackage{textcomp}
\usepackage{lipsum}
\usepackage{xcolor}
\usepackage{dsfont}
\usepackage{bm}
\usepackage{url}
\usepackage{placeins}
\usepackage[capitalise]{cleveref}
\usepackage{algorithm}
\usepackage[noend]{algpseudocode}
\usepackage{mleftright}

\newif\ifpreprint
\preprinttrue

\mleftright
\pgfplotsset{compat=newest}
\newcommand{\<}{\negmedspace{}}

\renewcommand{\Pr}{\operatorname{Pr}}

\newcommand{\mc}{\mathcal}
\DeclareMathOperator{\E}{\mathbb{E}}

\DeclareMathOperator{\Var}{Var}
\DeclareMathOperator{\supp}{supp}

\DeclarePairedDelimiter{\norm}{\lVert}{\rVert}

\DeclareMathOperator*{\argmax}{arg\,max}

\newtheorem{theorem}{Theorem}
\newtheorem{lemma}{Lemma}

\begin{document}
\bstctlcite{IEEEexample:BSTcontrol}

\title{Multiple Importance Sampling for Guided Speculative Inference}
\author{
    \IEEEauthorblockN{Joseph Rowan and Ashish J. Khisti}\\
    \IEEEauthorblockA{Department of Electrical and Computer Engineering, University of Toronto\\
    Email: joseph.rowan@mail.utoronto.ca, akhisti@ece.utoronto.ca}
}

\maketitle

\begin{abstract}
We study an extension of the \emph{Guided Speculative Inference} (GSI) algorithm \cite{geuter2025a} for test-time-scaling of large language models (LLMs) to a novel setting where an ensemble of smaller LLMs proposes draft reasoning steps which are later verified by a larger base model and a process reward model.
Using the established framework of multiple importance sampling (MIS), we introduce weighting schemes to aggregate samples from a diverse collection of proposal distributions.
We provide new bounds on the KL divergence from the optimal reward-tilted target distribution and the expected reward degradation, then demonstrate that the well-known \emph{balance heuristic} is optimal.
Compared to GSI with a single well-tuned drafter, combining several diverse LLMs while retaining the same total number of samples maintains competitive pass-at-$N$ downstream performance while being mathematically robust to situations where the support of a single draft model does not cover that of the target, where GSI may fail.
\end{abstract}

\section{Introduction}
Scaling computation in the pretraining phase, by increasing both the number of parameters and the size of training data, has driven remarkable improvements in the performance of large language models (LLMs) and achieved state-of-the-art results in many generative natural language tasks \cite{kaplan2020,hoffman2022}.
However, as limits to both computational resources and the data available for pretraining become apparent, \emph{test-time scaling} has emerged as a complementary approach for further improving the performance of an already-trained model by increasing the amount of computation at inference time \cite{muennighoff2025,snell2025,zhang2025a}.
While numerous approaches to test-time scaling exist, most may be categorized as either sequential, parallel or hybrid methods \cite{zhang2025a}.
In sequential approaches, a response is refined iteratively, with later computations building upon previous steps.
On the other hand, parallel methods generate several full responses in a single batch and return the best one as determined by an aggregation function.
Hybrid approaches attempt to combine the best of both worlds by generating multiple candidates at each iteration of a sequential process, then determining the most promising path by means of a search algorithm.

Best-of-$N$ (BoN) sampling is a straightforward scaling approach that can be used in both the parallel and hybrid regimes \cite{gao2023,song2025}.
It involves drawing $N$ samples from the model, evaluating a score for each, and choosing whichever sample has the highest score.
Despite its simplicity, BoN remains popular in practice by virtue of mapping readily to efficient prompt batching, and enjoys desirable theoretical properties while being competitive with more involved and computationally expensive test-time strategies \cite{beirami2025}.
As $N$ grows large however, the cost of BoN sampling may itself become prohibitive, which has motivated the search for methods that perform repeated sampling from a cheaper proxy model coupled with parallel verification by the original, large model to obtain comparable performance with reduced latency, analogous to speculative decoding in standard inference \cite{leviathan2023,liao2025,cemri2025,geuter2025a}.
\emph{Guided Speculative Inference} (GSI) \cite{geuter2025a} is a recent algorithm that uses a small LLM to generate candidates for verification, weighting them in such a way that the output distribution becomes close to the target.
However, as it relies on the fragile assumption that a single drafter's support can cover the target distribution, GSI remains vulnerable to lower acceptance rates during verification and degraded performance if the draft model is chosen poorly.
In this work we motivate the use of an ensemble of draft models to increase robustness, relaxing the coverage requirement to one on the \emph{union} of the drafters' supports by introducing a principled multiple importance sampling (MIS) \cite{elvira2019} framework for resampling the proposal sequences.
\ifpreprint\else
An extended version of the paper including all proofs is provided in \cite{extendedversion}.
\fi

\section{Problem Setup and Related Work}
In this work, following the setting considered in \cite{geuter2025a}, we deal with multi-step reasoning models that structure their response to a prompt as a sequence of \emph{reasoning steps} leading to a final answer.
Assuming a vocabulary $\mc{X}$, we let the set of all finite-length token sequences be $\mc{X}^\ast = \{ x = (x_1, \ldots, x_n) : n \in \mathbb{N}, x_i \in \mc{X} \}$ and the set of probability distributions on $\mc{X}^\ast$ be $\mc{P}(\mc{X}^\ast)$.
Given a prompt $c \in \mc{X}^\ast$, the model furnishes a distribution $\pi(\cdot \mid c) \in \mc{P}(\mc{X}^\ast)$ which can be used to sample the first reasoning step.
In practice, sampling from the block-level distribution $\pi(\cdot \mid c)$ is done autoregressively, outputting one token at a time in each forward pass.
After sampling the first reasoning step $X_1$, we append it to the prompt and sample subsequent blocks $X_2 \sim \pi(\cdot \mid c, X_1)$, $X_3 \sim \pi(\cdot \mid c, X_1, X_2)$ and so on until the end of the full output sequence, which will be signposted by an end-of-sequence token $\mathtt{EOS}$.
We assume the complete response comprises $T$ reasoning steps in total, where $T$ is finite, and the complete response is then the sequence $(X_1, \ldots, X_T)$.

Furthermore, to score the reasoning steps emitted by the LLM, we require the existence of a \emph{process reward model} (PRM), which acts as a step-level proxy for gauging the accuracy of and human preference for the model's outputs \cite{lightman2024}.
The PRM is modeled as a mapping $r : \mc{X}^\ast \times \mc{X}^\ast \to [0,R]$ which, given a potentially incomplete sequence of reasoning steps $(X_1, \ldots, X_t)$ and a prompt $c$, computes a reward $r(X_1, \ldots, X_t \mid c)$.
In practice, the PRM is typically implemented as an additional neural network, and is considered to be an approximation for an underlying true reward function or \emph{gold reward} $r^\ast : \mc{X}^\ast \times \mc{X}^\ast \to \mathbb{R}$ which perfectly captures the desirability of the output \cite{gao2023}.
We focus on devising and analyzing hybrid best-of-$N$ policies which, given $N$ blocks of tokens $X_t^{(1)}, \ldots, X_t^{(N)}$ computed at step $t$ and the associated rewards, aggregate these into a single sample $X_t$ which forms part of the input to the next reasoning step; the distribution over the policy's outputs is denoted as $\tilde{\pi}$.

\subsection{Hard BoN Sampling}
The \emph{hard BoN}, or simply BoN, strategy runs the PRM on each sample and outputs the sample with the maximum reward \cite{beirami2025}.
Formally, at time $1 \leq t \leq T$ and given previously chosen reasoning steps $X_1, \ldots, X_{t-1}$, we sample $X_t^{(1)}, \ldots, X_t^{(N)}$ iid from $\pi(\cdot \mid c, X_1, \ldots X_{t-1})$ and use the PRM to select
\begin{IEEEeqnarray}{c}
    X_t = \argmax_{1 \leq i \leq N} r(X_1, \ldots, X_{t-1}, X_t^{(i)} \mid c).
\end{IEEEeqnarray}
While this is a surprisingly effective model alignment strategy \cite{gao2023,eisenstein2024} when it comes to increasing the accuracy of a model's outputs, it can suffer from a phenomenon known as \emph{reward hacking}.
Since any real-world PRM is an imperfect approximation of the true reward, it is often possible for a model to exploit these imperfections and generate responses that score highly under the PRM but are otherwise low in quality \cite{skalse2022}.
As $N$ increases, hard BoN becomes more susceptible to such failure modes, limiting its ability to scale to large $N$ without eventually degrading performance \cite{huang2025a}.

\subsection{Soft BoN Sampling}
A popular and theoretically grounded approach to devising a policy is to solve a KL-regularized reward optimization problem over $\tilde{\pi}$, which may help prevent reward hacking compared to hard BoN \cite{verdun2025,aminian2026}.
Without loss of generality, we focus in all that follows on generating the first reasoning step $X_1$; for later steps, the partial response can be concatenated with the prompt $c$ and the same procedure repeated using the updated conditional distributions.
Furthermore, we let the conditioning on $c$ be implicit for brevity and omit it.
We solve
\begin{IEEEeqnarray}{c}
    \max_{\tilde{\pi} \in \mc{P}(\mc{X}^\ast)} \bigl[ \E_{X \sim \tilde{\pi}}[r(X)] - D_{\mathrm{KL}}(\tilde{\pi} \Vert \pi) / \beta \bigr]
\end{IEEEeqnarray}
and the objective is therefore to maximize the expected reward while keeping the policy close to the base model $\pi$ in a KL sense, with $\beta > 0$ controlling the strength of the KL regularization.
The optimal policy is to sample $X_1$ following the \emph{reward-tilted distribution}
\begin{IEEEeqnarray}{c}
    \pi^\ast(x) = \pi(x) e^{\beta r(x)} / Z \label{eqn:tilted_distribution}
\end{IEEEeqnarray}
where $Z = \E_{X \sim \pi}[e^{\beta r(X)}]$ \cite{korbak2022}.
However, apart from the special case of sampling exactly one token at a time, we do not have a direct way to sample exactly from \eqref{eqn:tilted_distribution}.
\emph{Soft BoN} sampling \cite{verdun2025} proposes to approximate the reward-tilted distribution by taking $N$ reasoning steps $X_1^{(1)}, \ldots, X_1^{(N)}$ iid from $\pi$, as in hard BoN, and selecting one probabilistically via importance resampling \cite{rubin1987,gelman2013}.
The $i$th sample is assigned a weight $W_i$ and an index $S$ is selected from $\{ 1, \ldots, N \}$ with probabilities proportional to $W_1, \ldots, W_N$ respectively.
The chosen block is then $X_1 = X_1^{(S)}$.
Since the samples are iid from $\pi$, the resampling weights should satisfy \cite{verdun2025}
\begin{IEEEeqnarray}{c}
    W_i \propto \frac{\pi^\ast(X_1^{(i)})}{\pi(X_1^{(i)})} = e^{\beta r(X_1^{(i)})} / Z.
\end{IEEEeqnarray}
The importance resampling step may also be framed as sampling based on the temperature-weighted softmax probability distribution of the rewards.
It was shown in \cite{verdun2025} that the KL divergence between the soft BoN policy and the reward tilted distribution satisfies
\begin{IEEEeqnarray}{c}
    D_{\mathrm{KL}}(\pi^\ast \Vert \tilde{\pi}) \leq \log\biggl( 1 + \frac{\Var_{X \sim \pi}(e^{\beta r(X)})}{N (\E_{X \sim \pi}[e^{\beta r(X)}])^2} \biggr) \label{eqn:sbon_kl_bound}
\end{IEEEeqnarray}
whereas hard BoN does not have such a guarantee.

\subsection{Guided Speculative Inference}\label{sec:gsi}
Inspired by speculative decoding, GSI \cite{geuter2025a} aims to accelerate soft BoN by sampling from a smaller draft model $\pi_s$, then verifying all the samples simultaneously with a larger model $\pi_b$ instead of sampling from the latter directly.
Since \eqref{eqn:tilted_distribution}, with the base model being $\pi = \pi_b$, can be rewritten as
\begin{IEEEeqnarray}{c}
    \pi^\ast(x) = \frac{\pi_s(x)}{Z} \exp\biggl( \beta r(x) + \log \frac{\pi_b(x)}{\pi_s(x)}  \biggr)
\end{IEEEeqnarray}
the tilted distribution may also be approximated by taking $N$ iid samples from $\pi_s$ and using the resampling weights
\begin{IEEEeqnarray}{c}
    W_i \propto \frac{\pi^\ast(X_1^{(i)})}{\pi_s(X_1^{(i)})} = \frac{1}{Z} \exp\biggl( \beta r(X_1^{(i)}) + \log \frac{\pi_b(X_1^{(i)})}{\pi_s(X_1^{(i)})}  \biggr). \IEEEeqnarraynumspace
\end{IEEEeqnarray}
However, we can only hope to approximate $\pi^\ast$ in this manner if the coverage assumption $\supp \pi_b \subseteq \supp \pi_s$ is satisfied.
Under this assumption, as was shown in \cite{geuter2025a},
\begin{IEEEeqnarray}{c}
    D_{\mathrm{KL}}(\pi^\ast \Vert \tilde{\pi}) \leq \log \biggl( \frac{e^{2 \beta R}(\chi^2(\pi_b \Vert \pi_s) + 1)}{N} + \frac{N-1}{N} \biggr). \label{eqn:gsi_kl_bound}
\end{IEEEeqnarray}

\section{Soft Best-of-N with Multiple Draft Models}
We now introduce a novel technique for extending GSI to an ensemble of draft models via MIS \cite{bugallo2017,elvira2019}.
Assume initially that there are $N$ draft models $\pi_s^{(1)}, \ldots, \pi_s^{(N)}$ and the number of draft models is therefore equal to the number of importance samples, which is also $N$.
It is straightforward to extend this to the case where the number of draft models is $K$ and $M$ samples are drawn from each, giving $N = MK$.
Later, we will consider the more general case where the sample allocation is not required to be uniform.
The main advantage of the MIS approach is that it allows the coverage assumption to be relaxed; by weighting the samples appropriately, one only requires 
\begin{IEEEeqnarray}{c}
    \supp \pi_b \subseteq \cup_{1 \leq i \leq N} \supp \pi_s^{(i)}.
\end{IEEEeqnarray}
In practice, this may help avoid the possibility of ending up in a situation where the importance sampling approximation fails to converge to the target distribution.
Furthermore, we expect to see prompts that favor different draft models; some prompts for example may suit a drafter fine-tuned on math while others may better fit one fine-tuned on code, making the idea of combining an ensemble of draft models appealing.

Besides MIS, any number of approaches could be used to combine the drafts.
For example, majority voting \cite{wang2023} might be one alternative choice.
In contrast, our method directly extends the GSI resampling strategy described in \cref{sec:gsi} by using standard MIS resampling weights \cite{wyman2023}
\begin{IEEEeqnarray}{rCl}
    W_i & \propto & \frac{m_i(X_1^{(i)}) \pi^\ast(X_1^{(i)})}{\pi_s^{(i)}(X_1^{(i)})} \\
    & = & \frac{m_i(X_1^{(i)})}{Z} \exp\biggl( \beta r(X_1^{(i)}) + \log \frac{\pi_b(X_1^{(i)})}{\pi_s^{(i)}(X_1^{(i)})} \biggr) \IEEEeqnarraynumspace
\end{IEEEeqnarray}
where $m_i$ is referred to as the $i$th \emph{MIS weight} and is restricted to be a function of $X_1^{(i)}$ only.
Otherwise, $m_i$ can have whatever form we like so long as it satisfies the following canonical requirements for a valid MIS scheme \cite{wyman2023,kondapaneni2019}:
\begin{IEEEeqnarray}{l}
    \pi_b(x) > 0 \implies \textstyle \sum_{i=1}^{N} m_i(x) = 1 \label{eqn:mis_req_1} \\
    \pi_s^{(i)}(x) = 0 \text{ and } \pi_b(x) > 0 \implies m_i(x) = 0. \label{eqn:mis_req_2}
\end{IEEEeqnarray}
After settling on a formula for the MIS weights, we propose the following procedure:
\begin{enumerate}
    \item Sample $N$ reasoning steps $X_1^{(1)}, \ldots, X_1^{(N)}$ independently from $\pi_s^{(1)}, \ldots, \pi_s^{(N)}$.
    \item Compute the MIS weights $m_i(X_1^{(i)})$ together with the resampling weights $W_i = m_i(X_1^{(i)}) g_i(X_1^{(i)})$ for $1 \leq i \leq N$ where
    \begin{IEEEeqnarray}{c}
        g_i(x) = \exp\biggl( \beta r(x) + \log \frac{\pi_b(x)}{\pi_s^{(i)}(x)} \biggr). \label{eqn:tilted_reward}
    \end{IEEEeqnarray}
    \item Let $\Pr[S = i \mid X_1^{(1)}, \ldots, X_1^{(N)}] = W_i / \sum_{j=1}^{N} W_j$ and output $X_1 = X_1^{(S)}$ as the selected reasoning step.
\end{enumerate}

\subsection{Theoretical Analysis}
We set out to show that the MIS schemes enjoy guarantees on the output's closeness to the optimal tilted distribution $\pi^\ast$, similar to the bounds for soft BoN and GSI in \eqref{eqn:sbon_kl_bound} and \eqref{eqn:gsi_kl_bound} respectively.
As a first step, we have the following lemma concerning the output distribution, analogous to \cite[Lem.~1]{verdun2025}.
When presenting our results on MIS, we also adopt the convention that $m_i(x) / \pi_s^{(i)}(x) = 0$ whenever both are zero.
We provide a proof in \ifpreprint \cref{sec:mis_output_dist_proof}. \else \cite[App.~A]{extendedversion}. \fi

\begin{lemma}\label{lem:mis_output_dist}
Let $\tilde{\pi}$ be the output distribution of a valid MIS scheme with $N$ samples.
Then, for all $x \in \supp \pi_b$,
\begin{multline}
    \tilde{\pi}(x) \geq \pi_b(x) e^{\beta r(x)} \sum_{i=1}^{N} m_i(x) \biggl( m_i(x) g_i(x) \\
    + \sum_{j \neq i} \E_{X \sim \pi_s^{(j)}}[m_j(X) g_j(X)] \biggr)^{-1}.
\end{multline}
\end{lemma}

Using Lemma~\ref{lem:mis_output_dist}, we can provide a bound on the KL divergence between  $\tilde{\pi}$ and the tilted distribution.
The proof is given in \ifpreprint \cref{sec:mis_kl_divergence_proof}. \else \cite[App.~B]{extendedversion}. \fi

\begin{theorem}\label{thm:mis_kl_divergence} 
For any valid MIS scheme, the KL divergence from the tilted distribution $\pi^\ast(x)$ satisfies
\begin{IEEEeqnarray}{c}
    D_{\mathrm{KL}}(\pi^\ast \Vert \tilde{\pi}) \leq \log\biggl( e^{2 \beta R} \sum_{x \in \mc{X}^\ast} \sum_{i=1}^{N} \frac{m_i(x)^2 \pi_b(x)^2}{\pi_s^{(i)}(x)} + 1 \biggr). \label{eqn:mis_kl_bound} \IEEEeqnarraynumspace
\end{IEEEeqnarray}
\end{theorem}

Moreover, one may be interested in bounding the expected reward degradation compared to sampling directly from $\pi^\ast$, i.e.\ how much the achieved reward is expected to decrease with respect to an idealized scenario in which exact samples from $\pi^\ast$ could be realized.
By applying Pinsker's inequality to the bound in Theorem~\ref{thm:mis_kl_divergence} in a similar approach to \cite{mroueh2025}, we arrive at the following result, with the proof found in \ifpreprint \cref{sec:mis_pinsker_proof}. \else \cite[App.~C]{extendedversion}. \fi

\begin{theorem}\label{thm:mis_pinsker}
The expected reward degradation satisfies
\begin{IEEEeqnarray}{rCl}
    \IEEEeqnarraymulticol{3}{l}{
        \E_{X \sim \pi^\ast}[r(X)] - \E_{X \sim \tilde{\pi}}[r(X)]
    }\nonumber\\* \
    & \leq & \frac{R}{\sqrt{2}} \biggl( \log\biggl( e^{2 \beta R} \sum_{x \in \mc{X}^\ast} \sum_{i=1}^{N} \frac{m_i(x)^2 \pi_b(x)^2}{\pi_s^{(i)}(x)} + 1 \biggr) \biggr)^{1/2}. \IEEEeqnarraynumspace
\end{IEEEeqnarray}
\end{theorem}

\subsection{Choices for the MIS Weights}
While the space of possible MIS weights is vast, we highlight two specific choices that have been widely adopted in the literature.

\subsubsection{The Choice of $m_i(X) = 1/N$}
While this choice does not necessarily respect the requirements \eqref{eqn:mis_req_1}--\eqref{eqn:mis_req_2} for a valid MIS scheme unless we impose the stricter coverage assumption that $\supp \pi_b \subseteq \supp \pi_s$ for all $i$, it remains popular due to its computational simplicity as neither additional proposal nor target evaluations are needed to compute the weights.
In the MIS taxonomy proposed in \cite{elvira2019}, this choice is referred to as scheme N1.
The KL bound in \eqref{eqn:mis_kl_bound} simplifies to
\begin{IEEEeqnarray}{c}
    D_{\mathrm{KL}}(\pi^\ast \Vert \tilde{\pi}) \leq \log \biggl( \frac{e^{2 \beta R}(\bar{\chi}^2 + 1)}{N} + 1 \biggr) \label{eqn:unif_kl_divergence}
\end{IEEEeqnarray}
where
\begin{IEEEeqnarray}{c}
    \bar{\chi}^2 = \frac{1}{N} \sum_{i=1}^{N} \chi^2(\pi_b \Vert \pi_s^{(i)})
\end{IEEEeqnarray}
is the mean $\chi^2$-divergence between the base and draft models.

\subsubsection{The Balance Heuristic}\label{sec:balance_heuristic}
Another popular choice is the \emph{balance heuristic}, also called scheme N3 in \cite{elvira2019}, which has long been known to be near-optimal when applying MIS to variance reduction in Monte Carlo integration \cite{veach1998}.
The balance heuristic uses $m_i(x) = \pi_s^{(i)}(x) / \sum_{j=1}^{N} \pi_s^{(j)}(x)$ and therefore requires significantly more computation since the probability of each sample must be evaluated under every proposal distribution.
We introduce the mixture distribution
\begin{IEEEeqnarray}{c}
    \psi(x) = \frac{1}{N} \sum_{i=1}^{N} \pi_s^{(i)}(x)
\end{IEEEeqnarray}
after which we can write $m_i(x) = \pi_s^{(i)}(x) / N \psi(x)$.
The KL divergence bound becomes
\begin{IEEEeqnarray}{c}
    D_{\mathrm{KL}}(\pi^\ast \Vert \tilde{\pi}) \leq \log \biggl( \frac{e^{2 \beta R} (\chi^2(\pi_b \Vert \psi) + 1)}{N} + 1 \biggr). \label{eqn:balance_kl_divergence}
\end{IEEEeqnarray}
Comparing \eqref{eqn:balance_kl_divergence} with \eqref{eqn:unif_kl_divergence} shows that the balance heuristic enjoys an upper bound that is at least as tight as the uniform weight scheme, by the convexity of the $\chi^2$ divergence.
In fact, we assert that it is the minimizer of the upper bound.

\begin{theorem}\label{thm:balance_optimality}
The balance heuristic minimizes the KL divergence bound in \eqref{eqn:mis_kl_bound} over all valid choices of MIS weights, provided that an equal number of samples is drawn from each draft model.
\end{theorem}

\begin{IEEEproof}
The proof makes use of a perturbation argument.
Let the balance heuristic weights be $\bar{m}_i(x)$.
For arbitrary valid MIS weights $m_i(x)$, we can always write $m_i(x) = \bar{m}_i(x) + \delta_i(x)$ and, as the weights must sum to one wherever $\pi_b(x) > 0$,
\begin{IEEEeqnarray}{c}
    \sum_{i=1}^{N} m_i(x) = \sum_{i=1}^{N} \bar{m}_i(x) + \sum_{i=1}^{N} \delta_i(x) \implies \sum_{i=1}^{N} \delta_i(x) = 0 \IEEEeqnarraynumspace
\end{IEEEeqnarray}
for any $x \in \mc{X}^\ast$ such that $\pi_b(x) > 0$.
Furthermore, by \eqref{eqn:mis_req_2}, $\delta_i(x) = 0$ for any such $x$ with $\pi_s^{(i)}(x) = 0$.
The only term in \eqref{eqn:mis_kl_bound} that depends on the choice of MIS weights is
\begin{IEEEeqnarray}{c}
    V \coloneq \sum_{x \in \mc{X}^\ast} \sum_{i=1}^{N} \frac{m_i(x)^2 \pi_b(x)^2}{\pi_s^{(i)}(x)}.
\end{IEEEeqnarray}
Substituting $m_i(x) = \bar{m}_i(x) + \delta_i(x)$ into $V$ gives
\begin{IEEEeqnarray}{rCl}
    V & = & \sum_{x \in \mc{X}^\ast} \sum_{i=1}^{N} \frac{(\bar{m}_i(x) + \delta_i(x))^2 \pi_b(x)^2}{\pi_s^{(i)}(x)} \\
    & = & \sum_{x \in \mc{X}^\ast} \sum_{i=1}^{N} \frac{\bar{m}_i(x)^2 \pi_b(x)^2}{\pi_s^{(i)}(x)} + 2 \sum_{x \in \mc{X}^\ast} \sum_{i=1}^{N} \frac{\bar{m}_i(x) \delta_i(x) \pi_b(x)^2}{\pi_s^{(i)}(x)} \nonumber \\
    && \< + \sum_{x \in \mc{X}^\ast} \sum_{i=1}^{N} \frac{\delta_i(x)^2 \pi_b(x)^2}{\pi_s^{(i)}(x)} \\
    & = & \bar{V} + 2 \sum_{x \in \mc{X}^\ast} \sum_{i=1}^{N} \frac{\bar{m}_i(x) \delta_i(x) \pi_b(x)^2}{\pi_s^{(i)}(x)} + \Delta
\end{IEEEeqnarray}
where $\bar{V}$ is the value of $V$ achieved by the balance heuristic and $\Delta \geq 0$.
Now, the formula for $\bar{m}_i(x)$ can be substituted into the second term to see that
\begin{IEEEeqnarray}{c}
    \sum_{i=1}^{N} \frac{\bar{m}_i(x) \delta_i(x) \pi_b(x)^2}{\pi_s^{(i)}(x)} = \frac{\pi_b(x)^2}{N \psi(x)} \sum_{i=1}^{N} \delta_i(x) = 0.
\end{IEEEeqnarray}
Thus we have shown that for any valid MIS scheme, $V \geq \bar{V}$.
However, we stress that minimizing the upper bound in \eqref{eqn:mis_kl_bound} does not mean that the KL divergence itself is minimized.
\end{IEEEproof}

\subsection{Optimizing the Sample Allocation}
Recent work has proposed training a router to interpret the prompt and dynamically decide which model should generate the samples in BoN framework with multiple LLMs \cite{geuter2025b}.
As an intuitive example, if a particular question involves mathematical reasoning, it should be assigned to a drafter fine-tuned on similar problems.
We can incorporate this notion into the MIS framework by adopting the established strategy of taking a different number of samples from each proposal distribution \cite{he2014,havran2014}, which can be viewed in our setting as a prior on which drafter will be more effective.
To be precise, let the total number of samples remain as $N$ but now we have $K \leq N$ draft models, where model $j$ contributes $N_j$ iid samples.

The results stated previously can immediately be extended to this new configuration, since it was not assumed that $\pi_s^{(1)}, \ldots, \pi_s^{(N)}$ had to be different.
To this end, let the proposal distributions be $p_1, \ldots, p_K$.
Let $\pi_s^{(i)} = p_j$ for all $N_{<j} + 1 \leq i \leq N_{<j} + N_j$, where $N_{<j} = \sum_{k < j} N_k$ is the total number of samples from all proposal distributions preceding $p_j$.
Theorems~\ref{thm:mis_kl_divergence} to~\ref{thm:balance_optimality} then all hold without modification.
Specifically, the tightest upper bound on the KL divergence remains \eqref{eqn:balance_kl_divergence} and is still achieved by the balance heuristic, with the mixture distribution now $\psi_\alpha(x) = \sum_{j=1}^{K} \alpha_j p_j(x)$.

Here, we let $\alpha = \begin{bmatrix} \alpha_1 & \cdots & \alpha_K \end{bmatrix}^T$ and note that in effect, $\alpha_j = N_j / N$ since $N_j$ samples are drawn from $p_j$.
The upper bound is minimized by making the $\chi^2$ divergence between the mixture distribution and the base model $\pi_b$ as small as possible; if we relax the problem by allowing $\alpha$ to take any value in the probability simplex, this becomes equivalent to the convex optimization problem
\begin{IEEEeqnarray}{lCl}
    \text{minimize} & \quad & \sum_{x \in \mc{X}^\ast} \frac{\pi_b(x)^2}{\sum_{j=1}^{K} \alpha_j p_j(x)} \\
    \text{subject to} & \quad & \sum_{i=1}^{K} \alpha_k = 1, \ \alpha_k \geq 0. 
\end{IEEEeqnarray}
However, the fixed uniform allocation described in \cref{sec:balance_heuristic} is likely to be of greater practical relevance.
First, the $\chi^2$ divergence over $\mc{X}^\ast$ cannot be computed exactly.
Second, the proposal and target distributions are expected to vary across different reasoning steps, making it difficult to apply adaptive importance sampling methods \cite{bugallo2017,he2014}.
Third, in many real-world systems the overall latency will be determined by that of the slowest model; if all drafters have a similar number of parameters, this will typically be whichever one has the largest batch size, favoring a uniform allocation.
Nevertheless, in heterogeneous environments where the drafters have different computational costs, varying the sample allocation may become useful.
We leave the study of this setting to future work.

\section{Experiments}\label{sec:experiments}
We perform experiments on the Qwen2.5 series \cite{qwen2.5} using Qwen2.5-Math \SI{7}{B} as the target model and Qwen2.5-Math, Instruct and Coder \SI{1.5}{B} as the draft models with a uniform sample allocation.
For GSI, all samples are from Qwen2.5-Math \SI{1.5}{B}.
The PRM is Qwen2.5-Math-PRM \SI{7}{B} \cite{zhang2025b}, which is fine-tuned on mathematics.
Each draft model is run on a separate L40S GPU, with 3 GPUs being used in total.
Our experiments focus on reasoning tasks and evaluate \SI{100} prompts each from the MMLU-Pro \cite{wang2024} and MATH-500 \cite{lightman2024} datasets, with each response capped at 40 steps and 4096 tokens.

\subsection{Hyperparameter Selection}
Following the choice made in \cite{geuter2025a}, we fix the KL regularization to be $\beta = 20$.
Moreover, for GSI, MIS-N1 and MIS-N3, we introduce a rejection threshold of $\alpha = 0.5$, which was also used in \cite{geuter2025a} to empirically increase reasoning performance.
If the tilted reward of the selected draft reasoning step is less than $\alpha$, we reject it and fall back to vanilla soft BoN sampling from the target model.
This step guards against rare cases where the draft model is misaligned with the target such that importance resampling does not reasonably approximate $\pi^\ast$.

\subsection{Results on Reasoning Tasks}
The primary metric for reasoning accuracy is the pass-at-$N$ rate, i.e.\ the fraction of prompts for which a correct answer is given when taking $N$ samples.
Across both datasets, test-time scaling with speculative inference increases downstream performance compared to one-shot reasoning using the target model alone, which achieves accuracies of 0.27 and 0.79 on MMLU-Pro and MATH-500 respectively.
As seen in \cref{tab:accuracy_results} however, increasing $N$ further provides relatively little gain, suggesting that sampling a small pool of reasoning steps is sufficient.
Even though the MIS strategies take only one-third as many samples compared to single-model GSI from the fine-tuned Qwen2.5-Math drafter, which is expected to align best with the chosen PRM and target model, they achieve comparable pass-at-$N$ rates across a variety of sample counts, with MIS-N3 in particular demonstrating strong performance.

The advantage of the MIS approach becomes more apparent when moving beyond the narrow domain of MATH-500 into the more diverse MMLU-Pro dataset.
While pass-at-$N$ rates can saturate due to imperfections in PRM-based scoring \cite{gao2023}, the tilted reward improves compared to GSI, as seen in \cref{fig:tilted_reward}, especially when $N$ is small.
MIS-N3 also consistently outperforms MIS-N1 in terms of tilted reward on MMLU-Pro, validating its theoretical optimality from Theorem~\ref{thm:balance_optimality} and demonstrating the advantage of combining diverse samples with principled weightings in more challenging reasoning settings.

\begin{table}[b]
    \centering
    \begin{tabular}{@{}rrrrrrr@{}}
        \toprule
        & \multicolumn{3}{c}{MMLU-Pro} & \multicolumn{3}{c}{MATH-500} \\
        $N$ & GSI & MIS-N1 & MIS-N3 & GSI & MIS-N1 & MIS-N3 \\
        \midrule
        3 & 0.40 & 0.39 & 0.42 & 0.88 & 0.84 & 0.82 \\ 
        6 & 0.40 & 0.42 & 0.42 & 0.86 & 0.85 & 0.85 \\
        12 & 0.44 & 0.49 & 0.45 & 0.86 & 0.86 & 0.85 \\
        24 & 0.49 & 0.45 & 0.48 & 0.87 & 0.84 & 0.86 \\
        48 & 0.44 & 0.43 & 0.52 & 0.89 & 0.82 & 0.89 \\
        96 & 0.47 & 0.42  & 0.44 & 0.89 & 0.84 & 0.87 \\
        192 & 0.52 & 0.51 & 0.44 & 0.87 & 0.83 & 0.87 \\
        \bottomrule
    \end{tabular}
    \caption{Pass-at-$N$ rates of the GSI, MIS-N1 and MIS-N3 speculative test-time scaling strategies.}
    \label{tab:accuracy_results}
    \vspace{-1ex}
\end{table}

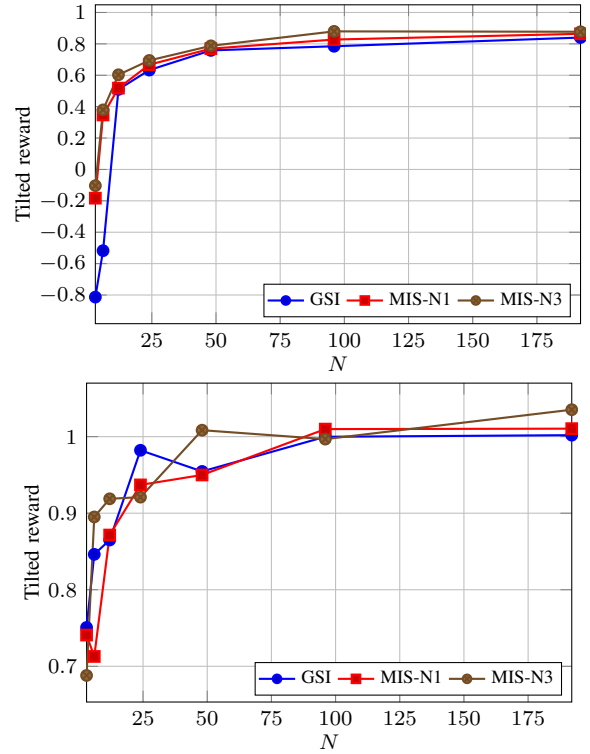
\begin{figure}[t]
    \centering
    \begin{tikzpicture}
    \begin{axis}[
        xlabel={\footnotesize $N$},
        ylabel={\footnotesize Tilted reward},
        width=8cm, height=5.8cm,
        xmin=3, xmax=192,
        xtick distance=25,
        ytick distance=0.2,
        tick label style={font=\footnotesize},
        legend style={at={(0.98,0.02)}, anchor=south east, cells={anchor=west}, inner xsep=1pt, inner ysep=1pt, font=\scriptsize},
        legend image post style={xscale=0.8, yscale=0.8},
        legend columns=3,
        xlabel near ticks,
        ylabel near ticks,
        scaled y ticks=false,
        label style={inner sep=0pt},
        yticklabel style={/pgf/number format/fixed},
        grid=major,
    ]
        \addplot+[thick] coordinates {
            (3,-0.813305) (6,-0.517633) (12,0.509803) (24,0.632951) (48,0.758369) (96,0.784751) (192,0.840245)
        };
        \addlegendentry{GSI}
        \addplot+[thick] coordinates {
            (3,-0.182890) (6,0.347097) (12,0.518109) (24,0.667203) (48,0.769437) (96,0.827524) (192,0.864143)
        };
        \addlegendentry{MIS-N1}
        \addplot+[thick] coordinates {
            (3,-0.103452) (6,0.379495) (12,0.603013) (24,0.694392) (48,0.787447) (96,0.879565) (192,0.877220)
        };
        \addlegendentry{MIS-N3}
    \end{axis}
\end{tikzpicture}
    
    \vspace{1ex}
    
    \begin{tikzpicture}
    \begin{axis}[
        xlabel={\footnotesize $N$},
        ylabel={\footnotesize Tilted reward},
        width=8cm, height=5.8cm,
        xmin=3, xmax=192,
        xtick distance=25,
        ytick distance=0.1,
        tick label style={font=\footnotesize},
        legend style={at={(0.98,0.02)}, anchor=south east, cells={anchor=west}, inner xsep=1pt, inner ysep=1pt, font=\scriptsize},
        legend image post style={xscale=0.8, yscale=0.8},
        legend columns=3,
        xlabel near ticks,
        ylabel near ticks,
        scaled y ticks=false,
        label style={inner sep=0pt},
        yticklabel style={/pgf/number format/fixed},
        grid=major,
    ]
        \addplot+[thick] coordinates {
            (3,0.75074) (6,0.846238) (12,0.864989) (24,0.982205) (48,0.954558) (96,0.999855) (192,1.001937)
        };
        \addlegendentry{GSI}
        \addplot+[thick] coordinates {
            (3,0.740795) (6,0.712852) (12,0.871471) (24,0.936958) (48,0.949945) (96,1.009863) (192,1.010448)
        };
        \addlegendentry{MIS-N1}
        \addplot+[thick] coordinates {
            (3,0.688038) (6,0.895113) (12,0.918814) (24,0.920978) (48,1.008443) (96,0.996784) (192,1.035213)
        };
        \addlegendentry{MIS-N3}
    \end{axis}
\end{tikzpicture}
    \caption{Average tilted reward achieved by draft responses on MMLU-Pro (above) and MATH-500 (below).}
    \label{fig:tilted_reward}
    \vspace{-1em}
\end{figure}

\subsection{Comparative Runtime}
We show the average time per reasoning step in \cref{tab:reasoning_times} for the most computationally expensive settings of $N = 96$ and $N = 192$; the overhead compared to GSI ranges from \SI{3}{\percent} in the best case to a modest \SI{10}{\percent} at worst, indicating that MIS-N3 with the balance heuristic remains  viable as a test-time scaling framework despite requiring additional proposal evaluations.

\begin{table}[b]
    \centering
    \begin{tabular}{@{}rrrrrrr@{}}
        \toprule
        & \multicolumn{3}{c}{MMLU-Pro} & \multicolumn{3}{c}{MATH-500} \\
        $N$ & GSI & MIS-N1 & MIS-N3 & GSI & MIS-N1 & MIS-N3 \\
        \midrule
        96 & 9.93 & 9.41 & 10.19 & 10.69 & 11.49 & 11.81 \\
        192 & 16.55 & 15.60 & 17.48 & 15.41 & 15.73 & 16.86 \\
        \bottomrule
    \end{tabular}
    \caption{Seconds per step for speculative test-time-scaling.}
    \label{tab:reasoning_times}
\end{table}

\section{Conclusion}
We have extended PRM-guided speculative reasoning to an ensemble setting by introducing a principled framework based on MIS, which satisfies new theoretical guarantees on the KL divergence from the optimal tilted distribution.
Our approach bypasses the constraint that a single drafter's support must fully cover the target distribution and maintains competitive pass-at-$N$ rates with comparable inference speed.
The theoretical results for general MIS schemes and the balance heuristic lay a foundation for future work which may incorporate routers to intelligently allocate compute across diverse experts, explore sources of drafter diversity other than fine-tuned variants of the same base model, or expand the framework beyond multi-step reasoning into domains such as creative writing or code generation.

\section{Acknowledgments}
Resources used for this
research were provided in part by the Province of Ontario, the Government of Canada
through CIFAR, and companies sponsoring the Vector Institute (\url{www.vectorinstitute.ai/partnerships/current-partners/}).

\bibliography{bibliography}

\ifpreprint
\newpage
\onecolumn
\appendix
\renewcommand{\theequation}{A\arabic{equation}}
\renewcommand{\thefigure}{A\arabic{figure}}
\setcounter{equation}{0}
\setcounter{figure}{0}

\subsection{Proof of Lemma~\ref{lem:mis_output_dist}}\label{sec:mis_output_dist_proof}
The proof extends that of \cite[Lem.~1]{verdun2025}, and we begin by directly computing the probability of choosing $X_1 = x$ as the output, which is
\begin{IEEEeqnarray}{rCl}
    \tilde{\pi}(x) & = & \sum_{i=1}^{N} \Pr[X_1 = x, S = i] \\
    & = & \sum_{i=1}^{N} \Pr[X_1^{(i)} = x, S = i] \\
    & = & \sum_{i \in \mc{S}(x)} \Pr[X_1^{(i)} = x, S = i] \\
    & = & \sum_{i \in \mc{S}(x)} \pi_s^{(i)}(x) \Pr[S = i \mid X_1^{(i)} = x]. \label{eqn:expanded_output}
\end{IEEEeqnarray}
where we have defined $\mc{S}(x) = \{ i \in \{1, \ldots, N \} : \pi_s^{(i)}(x) > 0 \}$.
The analysis will be the same for each $i$, so let us focus on computing $\Pr[S = 1 \mid X_1^{(1)} = x]$.
\begin{IEEEeqnarray}{rCl}
    \IEEEeqnarraymulticol{3}{l}{
        \Pr[S = 1 \mid X_1^{(1)} = x]
    }\nonumber\\* \quad
    & = & \sum_{x_2 \in \mc{X}^\ast} \cdots \sum_{x_N \in \mc{X}^\ast} \Pr[S = 1 \mid X_1^{(1)} = x, X_1^{(2)} = x_2, \ldots, X_1^{(N)} = x_N] \Pr[X_1^{(2)} = x_2, \ldots, X_1^{(N)} = x_N] \IEEEeqnarraynumspace \\
    & = & \sum_{x_2 \in \mc{X}^\ast} \cdots \sum_{x_N \in \mc{X}^\ast} \frac{m_1(x) g_1(x)}{m_1(x) g_1(x) + \sum_{j=2}^{N} m_j(x_j) g_j(x_j)} \prod_{j=2}^{N} \pi_s^{(j)}(x_j) \\
    & = & \E\biggl[ \frac{m_1(x) g_1(x)}{m_1(x) g_1(x) + \sum_{j=2}^{N} m_j(X_1^{(j)}) g_j(X_1^{(j)})} \biggr].
\end{IEEEeqnarray}
Generalizing to $S = i$ for any $1 \leq i \leq N$ and substituting into \eqref{eqn:expanded_output},
\begin{IEEEeqnarray}{rCl}
    \tilde{\pi}(x) & = & \sum_{i \in \mc{S}(x)} \pi_s^{(i)}(x) \E\biggl[ \frac{m_i(x) g_i(x)}{m_i(x) g_i(x) + \sum_{j \neq i} m_j(X_1^{(j)}) g_j(X_1^{(j)})} \biggr] \\
    & = & \sum_{i \in \mc{S}(x)} \pi_s^{(i)}(x) m_i(x) g_i(x) \E\biggl[ \biggl( m_i(x) g_i(x) + \sum_{j\neq i} m_j(X_1^{(j)}) g_j(X_1^{(j)}) \biggr)^{-1} \biggr] \\
    & \stackrel{\text{(a)}}{\geq} & \sum_{i \in \mc{S}(x)} \pi_s^{(i)}(x) m_i(x) g_i(x) \biggl( m_i(x) g_i(x) + \sum_{j \neq i} \E_{X \sim \pi_s^{(j)}}[m_j(X) g_j(X)] \biggr)^{-1} \\
    & = & \pi_b(x) e^{\beta r(x)} \sum_{i \in \mc{S}(x)} m_i(x) \biggl( m_i(x) g_i(x) + \sum_{j \neq i} \E_{X \sim \pi_s^{(j)}}[m_j(X) g_j(X)] \biggr)^{-1} \label{eqn:intermediate_prob_bound}
\end{IEEEeqnarray}
where (a) is by Jensen's inequality.
Because we consider a valid MIS scheme according to the conditions in \eqref{eqn:mis_req_1}--\eqref{eqn:mis_req_2}, $m_i(x) = 0$ for all $i \notin \mc{S}(x)$; adopting the convention that in such cases $m_i(x) / \pi_s^{(i)}(x) = 0$ gives $m_i(x) g_i(x) = 0$.
Moreover, for such an $i$,
\begin{IEEEeqnarray}{rCl}
    \sum_{j \neq i} \E_{X \sim \pi_s^{(j)}}[m_j(X) g_j(X)] & \geq & \sum_{j \in \mc{S}(x)} \E_{X \sim \pi_s^{(j)}}[m_j(X) g_j(X)] \\
    & \geq & \pi_b(x) e^{\beta r(x)} \sum_{j \in \mc{S}(x)} m_j(x) \\
    & \stackrel{\text{(a)}}{=} & \pi_b(x) e^{\beta r(x)} > 0
\end{IEEEeqnarray}
where (a) is again by \eqref{eqn:mis_req_1}.
Therefore, expanding the summation in \eqref{eqn:intermediate_prob_bound} to include all $1 \leq i \leq N$ does not affect the result, which yields
\begin{IEEEeqnarray}{c}
    \tilde{\pi}(x) \geq \pi_b(x) e^{\beta r(x)} \sum_{i=1}^{N} m_i(x) \biggl( m_i(x) g_i(x) + \sum_{j \neq i} \E_{X \sim \pi_s^{(j)}}[m_j(X) g_j(X)] \biggr)^{-1}.
\end{IEEEeqnarray}

\subsection{Proof of Theorem~\ref{thm:mis_kl_divergence}}\label{sec:mis_kl_divergence_proof}
Similar to the approach used to prove \cite[Th.~1]{geuter2025a}, we start from the definition of the KL divergence.
To avoid dealing with terms where $\pi_s^{(i)}(x) = 0$, we use the form of Lemma~\ref{lem:mis_output_dist} in \eqref{eqn:intermediate_prob_bound}, where the sum is restricted to $\mc{S}(x) = \{ i \in \{1, \ldots, N \} : \pi_s^{(i)}(x) > 0 \}$.
Then,
\begin{IEEEeqnarray}{rCl}
    \IEEEeqnarraymulticol{3}{l}{
        D_{\mathrm{KL}}(\pi^\ast \Vert \tilde{\pi})
    }\nonumber\\* \quad 
    & = & \sum_{x \in \mc{X}^\ast} \pi^\ast(x) \log \frac{\pi^\ast(x)}{\tilde{\pi}(x)} \\
    & \stackrel{\text{(a)}}{\leq} & \sum_{x \in \mc{X}^\ast} \frac{\pi_b(x) e^{\beta r(x)}}{\E_{X \sim \pi_b}[e^{\beta r(X)}]} \log\biggl( \frac{\pi_b(x) e^{\beta r(x)}}{\E_{X \sim \pi_b}[e^{\beta r(X)}]} \biggl[ \pi_b(x) e^{\beta r(x)} \sum_{i \in \mc{S}(x)} m_i(x) \biggl( m_i(x) g_i(x) \nonumber \\
    && \< + \sum_{j \neq i} \E_{X \sim \pi_s^{(j)}}[m_j(X) g_j(X)] \biggr)^{-1} \biggr]^{-1} \biggr) \\
    & = & \sum_{x \in \mc{X}^\ast} \frac{\pi_b(x) e^{\beta r(x)}}{\E_{X \sim \pi_b}[e^{\beta r(X)}]} \log\biggl( \frac{1}{\E_{X \sim \pi_b}[e^{\beta r(X)}]} \biggl[ \sum_{i \in \mc{S}(x)} m_i(x) \biggl( m_i(x) g_i(x) + \sum_{j \neq i} \E_{X \sim \pi_s^{(j)}}[m_j(X) g_j(X)] \biggr)^{-1} \biggr]^{-1} \biggr) \IEEEeqnarraynumspace \\
    & \stackrel{\text{(b)}}{\leq} & \log\biggl( \sum_{x \in \mc{X}^\ast} \frac{\pi_b(x) e^{\beta r(x)}}{(\E_{X \sim \pi_b}[e^{\beta r(X)}])^2} \biggl[ \sum_{i \in \mc{S}(x)} m_i(x) \biggl( m_i(x) g_i(x) + \sum_{j \neq i} \E_{X \sim \pi_s^{(j)}}[m_j(X) g_j(X)] \biggr)^{-1} \biggr]^{-1} \biggr) \label{eqn:intermediate_kl}
\end{IEEEeqnarray}
where (a) is by Lemma~\ref{lem:mis_output_dist} and (b) is by Jensen's inequality.
Next we bound
\begin{IEEEeqnarray}{rCl}
    \sum_{j \neq i} \E_{X \sim \pi_s^{(j)}}[m_j(X) g_j(X)] & = & \sum_{j \neq i} \E_{X \sim \pi_s^{(j)}} \biggl[ \frac{m_j(X) \pi_b(X) e^{\beta r(X)}}{\pi_s^{(j)}(X)} \biggr] \\
    & = & \adjustlimits \sum_{j \neq i} \sum_{x \in \supp \pi_s^{(j)}} m_j(x) \pi_b(x) e^{\beta r(x)} \\
    & \stackrel{\text{(a)}}{\leq} & \sum_{j \neq i} \E_{X \sim \pi_b} [m_j(X) e^{\beta r(X)}] \\
    & = & \sum_{j=1}^{N} \E_{X \sim \pi_b}[m_j(X) e^{\beta r(X)}] - \E_{X \sim \pi_b}[m_i(X) e^{\beta r(X)}] \\
    & \leq & \E_{X \sim \pi_b} \biggl[ e^{\beta r(X)} \sum_{j=1}^{N} m_j(X) \biggr] \stackrel{\text{(b)}}{=} \E_{X \sim \pi_b}[e^{\beta r(X)}]
\end{IEEEeqnarray}
where (a) follows by expanding the inner summation to be over all $x \in \mc{X}^\ast$, as all terms involved are nonnegative, and (b) makes use of the first canonical requirement for a valid MIS scheme stated in \eqref{eqn:mis_req_1}.
Going back to \eqref{eqn:intermediate_kl} and the KL divergence,
\begin{IEEEeqnarray}{rCl}
    \IEEEeqnarraymulticol{3}{l}{
        D_{\mathrm{KL}}(\pi^\ast \Vert \tilde{\pi})
    }\nonumber\\* \quad 
    & \leq & \log\biggl( \sum_{x \in \mc{X}^\ast} \frac{\pi_b(x) e^{\beta r(x)}}{(\E_{X \sim \pi_b}[e^{\beta r(X)}])^2} \biggl[ \sum_{i \in \mc{S}(x)} m_i(x) \biggl( \frac{m_i(x) \pi_b(x) e^{\beta r(x)}}{\pi_s^{(i)}(x)} + \E_{X \sim \pi_b}[e^{\beta r(X)}] \biggr)^{-1} \biggr]^{-1} \biggr) \\
    & \stackrel{\text{(a)}}{\leq} & \log\biggl( \sum_{x \in \mc{X}^\ast} \frac{\pi_b(x) e^{\beta r(x)}}{(\E_{X \sim \pi_b}[e^{\beta r(X)}])^2} \sum_{i \in \mc{S}(x)} m_i(x) \biggl( \frac{m_i(x) \pi_b(x) e^{\beta r(x)}}{\pi_s^{(i)}(x)} + \E_{X \sim \pi_b}[e^{\beta r(X)}] \biggr) \biggr) \\
    & = & \log\biggl( \sum_{x \in \mc{X}^\ast} \frac{\pi_b(x) e^{\beta r(x)}}{(\E_{X \sim \pi_b}[e^{\beta r(X)}])^2} \biggl( \sum_{i \in \mc{S}(x)} \frac{m_i(x)^2 \pi_b(x) e^{\beta r(x)}}{\pi_s^{(i)}(x)} + \E_{X \sim \pi_b}[e^{\beta r(X)}] \biggr) \biggr) \\
    & \stackrel{\text{(b)}}{\leq} & \log\biggl( e^{2 \beta R} \sum_{x \in \mc{X}^\ast} \sum_{i \in \mc{S}(x)} \frac{m_i(x)^2 \pi_b(x)^2}{\pi_s^{(i)}(x)} + 1 \biggr).
\end{IEEEeqnarray}
where (a) follows by the weighted arithmetic-harmonic mean inequality \cite{steele2004} and (b) uses the fact that $e^{\beta r(x)} \geq 1$ for all $x \in \mc{X}^\ast$.
While it is permitted that some of the $m_i(x)$'s be zero even when $i \in \mc{S}(x)$, for any valid MIS scheme the subset in which each $m_i(x) > 0$ must nevertheless form a partition of unity wherever $x \in \supp \pi_b$, as required by \eqref{eqn:mis_req_1}--\eqref{eqn:mis_req_2}.
Therefore the weighted arithmetic-harmonic mean inequality still holds, as the sum can be safely restricted to this subset.
Applying our convention that $m_i(x) / \pi_s^{(i)}(x) = 0$ when both the numerator and denominator are zero finally affords the form of the bound in \eqref{eqn:mis_kl_bound}.

\subsection{Proof of Theorem~\ref{thm:mis_pinsker}}\label{sec:mis_pinsker_proof}
To start, we have the following variational representation of the total variation distance between measures $P$ and $Q$ on $\mc{X}^\ast$, see for example \cite[Th.~7.7]{polyanskiy2023}:
\begin{IEEEeqnarray}{c}
    D_{\mathrm{TV}}(P, Q) = \frac{1}{2} \sup_{f \in \mc{F}}[\E_{X \sim P}[f(X)] - \E_{X \sim Q}[f(X)]]
\end{IEEEeqnarray}
where $\mc{F} = \{ f : \mc{X}^\ast \to \mathbb{R}, \norm{f}_\infty \leq 1 \}$.
Since $0 \leq r(x) \leq R$, we use the function
\begin{IEEEeqnarray}{c}
    f(x) = \frac{2}{R} \biggl( r(x) - \frac{R}{2} \biggr), \ x \in \mc{X}^\ast
\end{IEEEeqnarray}
which has $\norm{f}_\infty = 1$.
We get
\begin{IEEEeqnarray}{rCl}
    \IEEEeqnarraymulticol{3}{l}{
        \E_{X \sim \pi^\ast}[f(X)] - \E_{X \sim \tilde{\pi}}[f(X)]
    }\nonumber\\* \quad
    & = & \frac{2}{R} \sum_{x \in \mc{X}^\ast} \biggl( \biggl( r(x) - \frac{R}{2} \biggr) \pi^\ast(x) -\biggl( r(x) - \frac{R}{2} \biggr) \tilde{\pi}(x) \biggr) \\
    & = & \frac{2}{R} (\E_{X \sim \pi^\ast}[r(X)] - \E_{X \sim \tilde{\pi}}[r(X)]) + \sum_{x \in \mc{X}^\ast} (\tilde{\pi}(x) - \pi^\ast(x)) \\
    & = & \frac{2}{R} (\E_{X \sim \pi^\ast}[r(X)] - \E_{X \sim \tilde{\pi}}[r(X)]) \\
    & \leq & 2 D_{\mathrm{TV}}(\pi^\ast, \tilde{\pi}).
\end{IEEEeqnarray}
Therefore, applying Pinsker's inequality,
\begin{IEEEeqnarray}{c}
    \E_{X \sim \pi^\ast}[r(X)] - \E_{X \sim \tilde{\pi}}[r(X)] \leq R D_{\mathrm{TV}}(\pi^\ast, \tilde{\pi}) \leq R \sqrt{\frac{1}{2} D_{\mathrm{KL}}(\pi^\ast \Vert \tilde{\pi})}.
\end{IEEEeqnarray}
Bounding $D_{\mathrm{KL}}(\pi^\ast \Vert \tilde{\pi})$ using Theorem~\ref{thm:mis_kl_divergence} then gives the desired bound.
\fi

\end{document}